\documentclass[10pt]{article} 

\usepackage[accepted]{imports/rlj} 

\usepackage{amssymb}            
\usepackage{mathtools}          
\usepackage{graphicx}           

\title{Gated Q-learning: Add Off-Policy Bias to Taste}

\setrunningtitle{Gated Q-learning}

\author{Brett Daley\textsuperscript{1,$\dagger$}}

\emails{brett@brettdaley.ai}

\affiliations{
$^{1}$\textbf{Meta}
\par 
$^\dagger$ All experiments, data collection, and processing activities were conducted in the author's personal capacity. No experiments, data collection, or processing activities were conducted on Meta infrastructure.
}

\contribution{
    We identify a new, general class of \Qlambda algorithms that utilize state-action-dependent trace-decay values.
    This offers a novel perspective on partial off-policy bias correction in \Qlearning methods, where importance sampling cannot be applied.
}
{
    State-action-dependent traces for Expected Sarsa have been previously studied to control off-policy bias in conjunction with importance sampling \citep[e.g.,][Sec.~12.8]{munos2016safe,sutton2018reinforcement}.
    To the best of our knowledge, this idea has never been explored in \Qlearning beyond Watkins' \Qlambda \citep{watkins1989learning}, which applies exploration-conditional trace cuts to eliminate off-policy bias.
}
\contribution{
    We propose Gated \Qlambda, which implements soft, exploration-conditional trace cuts to mitigate off-policy bias while preserving multistep credit assignment.
    We very briefly discuss an \nsteptext version as well.
    }
    {
    Gated \Qlambda interpolates smoothly between the classic methods of Watkins' \Qlambda \citep{watkins1989learning} and Peng's \Qlambda \citep{peng1996incremental}.
}
\contribution{
    We conduct a large-scale hyperparameter sweep in a random-walk environment adapted for off-policy control, generating detailed heatmaps to visualize the influence of step size, trace decay, and gating on Gated \Qlambda.
    Our heatmaps clearly illustrate a performance trade-off between Watkins' \Qlambda and Peng's \Qlambda.
    }
    {
    \citet[][Example~7.1]{sutton2018reinforcement} describe the 19-state random walk that we adapt for our experiment.
}
\contribution{
    We derive the value-function operator underlying this general class of \Qlambda algorithms with state-action-dependent traces, formally proving how the chosen traces impact its contraction rate and fixed point.
    }
    {
    \citet{kozuno2021revisiting} derived similar results for Peng's \Qlambda.
    Our theorems significantly generalize these to the newly identified class of \Qlambda algorithms.
}


\summary{
    Multistep credit assignment is critical for sample-efficient reinforcement learning, yet managing off-policy bias in \Qlearning remains a fundamental challenge.
For 30 years, practitioners have been limited to a binary choice:
eliminate the bias at the cost of severely truncated eligibility traces (Watkins' \Qlambda), or ignore the bias to learn faster while injecting detrimental errors into the value estimates (Peng's \Qlambda).
Modern off-policy estimators fail to resolve this tension, as importance-sampling ratios collapse under \Qlearning's greedy target policy.
We introduce Gated \Qlearning, a novel algorithmic framework that ends this dilemma by smoothly interpolating between the two historical extremes.
Rather than relying on importance sampling, our approach employs a continuous, state-action-dependent gating mechanism to selectively attenuate eligibility traces in an exploration-aware manner.
We provide a rigorous theoretical foundation for this mechanism, proving that the expected operator remains a contraction mapping and deriving its exact fixed point.
Empirical evaluations verify that intermediate gating safely enables longer credit-assignment horizons, yielding faster initial learning than either extreme.
Gated \Qlearning offers a simple alternative to importance sampling while enabling customization of the effective multistep horizon and the amount of off-policy bias in \Qlearning agents.

}

\def\eqref#1{equation~\ref{#1}}

\def\1{\bm{1}}

\def\vone{{\bm{1}}}

\def\vc{{\bm{c}}}

\def\vq{{\bm{q}}}
\def\vr{{\bm{r}}}

\def\vv{{\bm{v}}}

\def\vx{{\bm{x}}}
\def\vy{{\bm{y}}}

\def\mE{{\bm{E}}}

\def\mI{{\bm{I}}}
\def\mJ{{\bm{J}}}

\def\mP{{\bm{P}}}

\def\mZ{{\bm{Z}}}

\def\mLambda{{\bm{\Lambda}}}

\DeclareMathAlphabet{\mathsfit}{\encodingdefault}{\sfdefault}{m}{sl}
\SetMathAlphabet{\mathsfit}{bold}{\encodingdefault}{\sfdefault}{bx}{n}

\def\emLambda{{\Lambda}}

\newcommand{\E}{\mathbb{E}}

\newcommand{\R}{\mathbb{R}}

\DeclareMathOperator*{\argmax}{arg\,max}

\usepackage{paper/preamble}

\begin{document}

\makeCover  
\maketitle  

\begin{abstract}
    Multistep credit assignment is critical for sample-efficient reinforcement learning, yet managing off-policy bias in \Qlearning remains a fundamental challenge.
For 30 years, practitioners have been limited to a binary choice:
eliminate the bias at the cost of severely truncated eligibility traces (Watkins' \Qlambda), or ignore the bias to learn faster while injecting detrimental errors into the value estimates (Peng's \Qlambda).
Modern off-policy estimators fail to resolve this tension, as importance-sampling ratios collapse under \Qlearning's greedy target policy.
We introduce Gated \Qlearning, a novel algorithmic framework that ends this dilemma by smoothly interpolating between the two historical extremes.
Rather than relying on importance sampling, our approach employs a continuous, state-action-dependent gating mechanism to selectively attenuate eligibility traces in an exploration-aware manner.
We provide a rigorous theoretical foundation for this mechanism, proving that the expected operator remains a contraction mapping and deriving its exact fixed point.
Empirical evaluations verify that intermediate gating safely enables longer credit-assignment horizons, yielding faster initial learning than either extreme.
Gated \Qlearning offers a simple alternative to importance sampling while enabling customization of the effective multistep horizon and the amount of off-policy bias in \Qlearning agents.

\end{abstract}

\section{Introduction}

Despite its simplicity, \Qlearning \citep{watkins1989learning} remains a staple of modern reinforcement learning (RL).
The appeal of \Qlearning lies in its theoretical elegance and its decoupling of the behavior policy from the target policy, allowing agents to continuously refine estimates of optimal values while exploring the environment or learning from historical replay buffers.
In deep RL specifically, where neural networks serve as function approximators, \Qlearning underpins the success of some of the most sample-efficient methods to date, including Deep Q-Networks \citep[DQN;][]{mnih2015human}, Rainbow \citep{hessel2018rainbow}, and Parallelized Q-Networks \citep[PQN;][]{gallici2025simplifying}.
It also plays a critical role in offline RL \citep[e.g.,][]{fujimoto2019off,kumar2019stabilizing,kumar2020conservative,kostrikov2022offline}, where its off-policy nature is ideal for learning from static datasets.
Consequently, advancing the algorithmic foundations of \Qlearning directly translates to broader improvements across a vast array of deep RL architectures.

Standard \Qlearning is rooted in \onesteptext temporal-difference (TD) learning \citep{sutton1988learning}, which struggles to assign credit quickly over long time horizons.
Multistep learning is crucial for accelerating this process, but the naive application of forward-view return estimators such as \nsteptext returns or \lambdareturns is strongly biased in off-policy settings.
This bias stems from the distributional mismatch between the agent's exploratory behavior and the targeted greedy behavior.
The theoretically correct approach is to eliminate this bias by truncating the multistep return estimates whenever an exploratory action is taken, as in Watkins' \Qlambda \citep{watkins1989learning}.
However, empirical evidence indicates that simply \emph{ignoring} these corrections often yields superior performance \citep{hernandez2018understanding,daley2019reconciling}, which is consistent with the motivation behind Peng's \Qlambda \citep{peng1996incremental}.
Practitioners are thus faced with a limited choice:
either strictly eliminate the bias at the cost of severe truncation and slower learning, or accept the bias and ultimately limit the length of multistep returns that can be safely deployed.

Although multistep off-policy estimators that enable finer-grained control over off-policy bias do exist, this control relies on importance sampling \citep{kahn1953methods} and is therefore not compatible with \Qlearning.
Key examples include Tree Backup \citep{precup2000eligibility}, Retrace \citep{munos2016safe}, and Recency-Bounded Importance Sampling \citep{daley2023trajectory}---all of which adapt the degree of reinforcement based on the target policy's action probabilities.
However, in \Qlearning, the target policy is strictly greedy, causing the importance-sampling ratio to vanish for every non-greedy action.
Consequently, this whole class of methods degenerates into the same aggressive Watkins-style correction, which, as previously mentioned, fails to preserve the long credit-assignment horizons needed for fast learning.

There is a clear need for a new mechanism to regulate off-policy bias in multistep \Qlearning, without the use of importance sampling.
We propose a novel approach that utilizes adaptive $\lambda$-values to partially ``gate'' the propagation of the eligibility trace specifically when exploratory (non-greedy) actions are taken.
This algorithm, which we call \emph{Gated \Qlambda}, mitigates some but not all of the off-policy bias while preserving the trace along greedy trajectories.
This strategy interpolates smoothly between the principled (but slow) Watkins' update and the biased (but fast) Peng's update.
We hypothesize that balancing this trade-off leads to superior learning compared to either extreme.
The name is borrowed from the gating mechanisms of recurrent and attention architectures (see \Cref{sub:gated_qlambda}), although its role here is distinct---it modulates credit assignment in RL.

Our paper is dedicated to deeply understanding the properties and implications of this gating mechanism in off-policy credit assignment.
We primarily focus on eligibility traces and \lambdareturns, although we briefly discuss an \nsteptext variant as well (see \Cref{sub:nstep_gated_qlearning}).
We first derive Gated \Qlambda as a special case of a new, more general \Qlambda class which permits state-action-dependent $\lambda$-values while targeting a greedy policy---unlike \citeauthor{munos2016safe}'s \citeyearpar{munos2016safe} per-decision operator, our class generalizes Peng's \Qlambda rather than naive \Qlambda.
This greatly broadens the scope of our theoretical analysis while helping to contextualize and justify the specific choice of our adaptive gating strategy.
We then conduct a focused hyperparameter study in a random walk to illustrate how the gating mechanism impacts credit assignment and learning speed.
Building upon these empirical insights, we formally analyze Gated \Qlambda to establish its contraction rate and fixed point, providing clear theoretical proof of the trade-off between trace preservation and off-policy bias.
Our results demonstrate that there are still fundamental \Qlearning improvements to be discovered, and that faster initial learning can be achieved by adjusting off-policy bias to a desired, intermediate amount.

\section{Background}

The RL problem considers an agent whose objective is to learn to act in its environment in a way that maximizes its expected cumulative discounted reward.
As in the foundational work of \citet{watkins1989learning}, which introduced \Qlearning, the RL problem is most commonly framed as solving a Markov Decision Process (MDP) from sample-based interaction.
The MDP is typically described by a tuple $(\S,\A,p,r)$.
At each discrete time step $t \geq 0$, the agent observes a state $S_t \in \S$ and selects an action $A_t \in \A$ according to a behavior policy $b(a|s)$, which maps states to probability distributions over actions.
The environment then transitions to a new state $S_{t+1}$ according to the transition dynamics $p(s' \mid s, a)$, and the agent receives a scalar reward $R_{t+1}$ governed by the reward function $r(s,a)$.
The fundamental objective of the agent is to maximize the expected value of the return, defined as the cumulative sum of discounted future rewards, $G_t \coloneqq \sum_{i=0}^{\infty} \gamma^i R_{t+i+1}$, where $\gamma \in [0, 1)$ is the discount factor.
In off-policy learning, we explicitly distinguish between this behavior policy (which explores the environment and generates trajectory data) and the target policy $\pi(a|s)$, the distinct policy that the algorithm is attempting to evaluate or optimize.

To measure the quality of a policy $\pi$, we define the action-value function $q_\pi(s, a) \coloneqq {\E_\pi [G_t \mid S_t = s, A_t = a]}$, which represents the expected return for taking action $a$ in state $s$ and subsequently following $\pi$.
\Qlearning aims to learn the optimal action-value function $q_*$ by representing these estimates as a tabular matrix or parameterized function $Q \in \RSA$.
The distinguishing feature of \Qlearning is that its target policy is always \emph{greedy} with respect to its current value estimates.
The estimated value of a state is therefore given by
\begin{equation*}
    V(s) \coloneqq \max_{a \in \A} Q(s,a)
    \,,
\end{equation*}
where the value of a terminal state is always defined to be $0$.
The classic \onesteptext \Qlearning update rule is then defined as
\begin{equation*}
    Q(S_t,A_t) \gets Q(S_t,A_t) + \alpha \Bigl(
            \underbrace{R_{t+1} + \gamma V(S_{t+1}) - Q(S_t,A_t)}_{\delta'_t}
        \Bigr)
    \,,
\end{equation*}
where $\alpha \in (0, 1]$ is the step size.
We refer to the quantity $\delta'_t$ as the \emph{\Qlearning (QL) error} to differentiate it from the classic state-value TD error.
Rooted in the Bellman optimality equation \citep{bellman1957dynamic}, this \onesteptext update is highly robust and guaranteed to converge to $q_*$ in tabular settings \citep{watkins1992q}.
However, because information about future rewards is solely conveyed through immediate bootstrapping, its \onesteptext nature results in painfully slow credit assignment, as rewards must propagate backward through the state-action space one step at a time over repeated episodic interactions.

To overcome the slow credit assignment, we can consider \emph{multistep} versions of \Qlearning by substituting the \onesteptext target with a generalized multistep return estimator $\hat{G}_t$:
\begin{equation*}
    Q(S_t,A_t) \gets Q(S_t,A_t) + \alpha \Bigl(
        \hat{G}_t - Q(S_t,A_t)
    \Bigr)
    \,.
\end{equation*}
Defining $\hat{G}_t$ to safely accelerate learning in off-policy settings has proven to be highly nontrivial.
If one were to simply use eligibility traces to accumulate a fading record of recent \onesteptext errors, analogous to on-policy TD($\lambda$) \citep{sutton1988learning}, it would produce the following forward-view target:
\begin{equation*}
    G^{\lambda\,\text{(naive)}}_t \coloneqq Q(S_t,A_t) + \sum_{i=0}^\infty (\gamma \lambda)^i \delta'_{t+i}
    \,.
\end{equation*}
This update has become known as \emph{naive} \Qlambda \citep[][Sec.~7.6]{sutton1998introduction} because it completely fails to address the off-policy distributional mismatch between the exploratory behavior policy and the greedy target policy.
As a result, it suffers from severe bias and is only guaranteed to converge in off-policy settings when $\lambda$ is kept impractically small \citep{harutyunyan2016q}, which heavily counteracts the original multistep benefits.

\citet{watkins1989learning} recognized that off-policy bias can be completely avoided by cutting the eligibility trace whenever a non-greedy action is taken.
Define the following greedy indicator function:
\begin{equation*}
    \greedy(s,a) \coloneqq a = \argmax_{a' \in \A} Q(s,a')
    \,,
\end{equation*}
where ties are broken arbitrarily but consistently.
Written as a forward-view return, Watkins' \Qlambda, which cuts traces, generates the following target:
\begin{align}
    \label{eq:watkins}
    G^{\lambda\,\text{(Watkins)}}_t &\coloneqq Q(S_t,A_t) + \sum_{i=0}^\infty \gamma^i \left(\prod_{j=1}^i \lambda_{t+j}\right) \delta'_{t+i}
    \,,\\
    \nonumber
    \text{where} \quad \lambda_t &\coloneqq \begin{cases}
        \lambda & \text{if } \greedy(S_t,A_t), \\
        0 & \text{otherwise,}
    \end{cases}
\end{align}
and $\prod_{j=1}^0 \lambda_{t+j} \coloneqq 1$ to correctly initialize the first weight.
Watkins' \Qlambda fully eliminates off-policy bias, and in this sense is the theoretically ``correct'' implementation of \Qlearning with eligibility traces.
Unfortunately, because exploratory actions are frequent during training, traces are cut early and often.
This severe truncation mostly counteracts the speed benefits of multistep learning, making it seem as though \Qlearning is simply at odds with the goals of extended credit assignment.

To circumvent the severe trace cutting of Watkins' method, \citet{peng1996incremental} introduced a modified version of \Qlambda that prioritizes rapid credit assignment over strict off-policy correctness.
To formalize this, we first define the state-value TD error as 
\begin{equation*}
    \delta_t \coloneqq R_{t+1} + \gamma V(S_{t+1}) - V(S_t)
    \,.
\end{equation*}
Then, Peng's \Qlambda target becomes a hybrid of \Qlearning and TD($\lambda$):
\begin{equation*}
    G^{\lambda\,\text{(Peng)}}_t \coloneqq Q(S_t,A_t) + \delta'_t + \sum_{i=1}^\infty (\gamma \lambda)^i \delta_{t+i}
    \,.
\end{equation*}
Unlike Watkins' \Qlambda, this estimator never cuts the eligibility trace.
While this makes the update strongly biased in off-policy settings, it is empirically much faster.
Furthermore, this estimator satisfies an elegant recursive equation:
\begin{equation*}
    G^{\lambda\,\text{(Peng)}}_t = R_{t+1} + \gamma \Bigl((1-\lambda) V(S_{t+1}) + \lambda G^{\lambda\,\text{(Peng)}}_{t+1} \Bigr)
    \,,
\end{equation*}
where the recursion is initialized by $V(S_{T})$ at the end of a truncated trajectory or $0$ at the end of an episode.
This makes it very efficient to compute over offline trajectories of experience.
As a result, Peng's \Qlambda has become a popular choice in trajectory-based deep RL \citep[e.g.,][]{harb2016investigating,mousavi2017applying,daley2019reconciling,kozuno2021revisiting,gallici2025simplifying,elelimy2025deep} and is the most common multistep alternative to the widely used \nsteptext return \citep[e.g.,][]{hessel2018rainbow}.
However, in both cases, relying on uncorrected bias is fundamentally flawed, as it ultimately limits the safe effective horizon of the multistep return.

In a separate line of research, modern methods for off-policy learning have successfully managed this bias-variance trade-off using variations of importance sampling \citep{kahn1953methods}.
These include algorithms such as Tree Backup \citep{precup2000eligibility}, Retrace \citep{munos2016safe}, and Recency-Bounded Importance Sampling \citep{daley2023trajectory}.
While highly effective for off-policy learning, these approaches inherently fall under the Sarsa class of algorithms, meaning they evaluate or optimize stochastic (non-greedy) target policies.
This effectiveness is fundamentally lost in \Qlearning due to its greedy target policy, which causes importance-sampling ratios to collapse to zero for every non-greedy action.
This structural limitation precludes the use of modern importance-sampling methods, including resampling \citep[e.g.,][]{schlegel2019importance}, severely limiting the degree to which off-policy corrections can be controlled in \Qlearning.
As a consequence, there has been surprisingly little progress on multistep \emph{\Qlearning} estimators, leaving practitioners caught between the two extremes of Watkins' and Peng's \Qlambda.

\section{Variable \Qlambda}
\label{sec:variable_qlambda}

Before introducing our Gated \Qlambda, we begin by formalizing a more general class of \emph{variable} \Qlambda methods, of which our algorithm is a special case.
When applying multistep \Qlearning, practitioners have historically been faced with a limited choice: either fully correct the off-policy bias as in Watkins' \Qlambda, or ignore the bias as in Peng's \Qlambda.
Furthermore, an ideal multistep \Qlearning method must have several properties that make it practically useful:
most notably the ability to regulate this bias without sacrificing computational efficiency.
To unify these different approaches, we adopt a variable $\lambda$ framework, which serves as a prime but underexplored candidate for resolving this dilemma.

The concept of state-action-dependent trace parameters has been discussed by \citet[][Sec.~12.8]{sutton2018reinforcement}, but only for \emph{Sarsa} methods and not \Qlearning.
In our notation, the \Qlearning return is
\begin{equation}
    \label{eq:lambda_variable}
    \tilde{G}^\lambda_t \coloneqq Q(S_t,A_t) + \delta'_t + \sum_{i=1}^\infty \gamma^i \left(\prod_{j=1}^i \lambda_{t+j}\right) \delta_{t+i}
    \,,
\end{equation}
where the overloaded function $\lambda \colon \SA \to [0,1]$ now determines $\lambda_t \coloneqq \lambda(S_t,A_t)$.
This return admits a recursive form:
\begin{equation}
    \label{eq:recursive_general}
    \tilde{G}^\lambda_t = R_{t+1} + \gamma \Bigl((1-\lambda_{t+1}) V(S_{t+1}) + \lambda_{t+1} \tilde{G}^\lambda_{t+1} \Bigr)
    \,.
\end{equation}
The equivalence between \Cref{eq:lambda_variable} and \Cref{eq:recursive_general} follows from a slight generalization of the fixed-$\lambda$ derivation given by \citet[][Sec.~5.2]{daley2025multistep}, and it is exactly this recursion that makes the return target efficient to compute, whether with eligibility traces or replayed trajectories.
The closest instance of this equation that we found is Eq.~12.20 of \citet{sutton2018reinforcement}, where it is given for Expected Sarsa but not \Qlearning.
The key distinction is how the target policies are defined (non-greedy for Expected Sarsa, greedy for \Qlearning), which in turn changes the definition of $V(S_{t+1})$ as well as the applicability of importance sampling.
To our knowledge, Watkins' \Qlambda is the only existing \Qlearning algorithm that leverages this particular variable-trace formulation to manage off-policy data, leaving its broader potential for fine-grained bias control entirely unexamined.\footnote{
    To use this formula with any choice of $\lambda(s,a)$ that preserves traces after exploratory actions, we must accept some bias and sacrifice convergence to $q_*$.
}

\begin{wraptable}{r}{0.5\textwidth}
    \caption{Comparison of different \Qlambda methods that can be expressed by \Cref{alg:gated_qlambda}.}
    \label{tab:qlambda}
    \begin{center}
    \setlength{\tabcolsep}{3pt}
    \begin{tabular}{ll}
    \toprule
    \textbf{Method} & \textbf{Decay Strategy} \\
    \midrule
    Watkins' \Qlambda & $\lambda_t \coloneqq \begin{cases} \lambda & \text{if } \greedy(S_t,A_t) \\ 0 & \text{otherwise} \end{cases}$ \\
    \midrule
    Peng's \Qlambda & $\lambda_t \coloneqq \lambda$ \\
    \midrule
    \textbf{Gated \Qlambda} & $\lambda_t \coloneqq \begin{cases} \lambda & \text{if } \greedy(S_t,A_t) \\ \lambda \chi & \text{otherwise} \end{cases}$ \\
    \bottomrule
    \end{tabular}
    \end{center}
\end{wraptable}

Crucially, this framework reveals that both Watkins' and Peng's \Qlambda are special cases of variable \Qlambda in \Cref{eq:lambda_variable}.
Peng's \Qlambda is achieved by simply substituting a constant value of $\lambda$.
Watkins' \Qlambda is less obvious;
it follows because $\lambda_{t+j}$ is nonzero only if $A_{t+j}$ is greedy, and therefore $\delta'_{t+i} = \delta_{t+i}$ wherever its weight in \Cref{eq:watkins} is nonzero.
We present pseudocode for the universal eligibility-trace template in \Cref{alg:gated_qlambda}, offering the specific $\lambda$ definitions for the methods in \Cref{tab:qlambda}.
In line~9, we indicate the specific choice of $\lambda$ to produce our flagship algorithm Gated \Qlambda that is introduced in the next section, but this line can be modified to implement Watkins' \Qlambda, Peng's \Qlambda, and any other variable \Qlambda method described by \Cref{eq:lambda_variable}.

\begin{algorithm}[t]
\caption{Gated \Qlambda}
\label{alg:gated_qlambda}
\SetAlgoLined
\DontPrintSemicolon
Initialize $Q(s,a)$ arbitrarily and $Z(s,a) \gets 0$ for all $(s,a)$\;
\BlankLine
\For{$t = 0, 1, \dots$}{
    Compute greedy action $A^*_t \coloneqq \argmax_{a \in \A} Q(S_t,a)$\;
    Take action $A_t$ according to the policy in state $S_t$ \tcp*{Reuse $A^*_t$ if needed}
    Observe reward $R_{t+1}$ and next state $S_{t+1}$\;
    \BlankLine
    $\nstep{1}_t \coloneqq \begin{cases}
        R_{t+1} & \text{if } \terminal(S_{t+1}) \\
        R_{t+1} + \gamma \max\limits_{a' \in \A} Q(S_{t+1},a') & \text{otherwise}
    \end{cases}$ \tcp*{\onesteptext return}
    $\delta'_t \coloneqq \nstep{1}_t - Q(S_t,A_t)$ \tcp*{QL error}
    $\delta_t \coloneqq \nstep{1}_t - Q(S_t,A^*_t)$ \tcp*{TD error}
    \BlankLine
    $\lambda_t \coloneqq \begin{cases}
        \lambda & \text{if } A_t = A^*_t \\
        \lambda \chi & \text{otherwise}
    \end{cases}$ \tcp*{Change to implement other algorithms}
    \BlankLine
    \ForEach{$(s,a)$}{
        $Z(s,a) \gets \gamma \lambda_t Z(s,a)$ \tcp*{Decay trace}
        $Q(s,a) \gets Q(s,a) + \alpha \delta_t Z(s,a)$ \tcp*{Apply TD error to past}
    }
    \BlankLine
    $Q(S_t,A_t) \gets Q(S_t,A_t) + \alpha \delta'_t$ \tcp*{Apply QL error to present}
    \BlankLine
    \uIf{$\terminal(S_{t+1})$}{
        $Z(s,a) \gets 0$ for all $(s,a)$ \tcp*{Reset traces}
        Reset environment state $S_{t+1}$ \tcp*{Next episode}
    }
    \Else{
        $Z(S_t,A_t) \gets Z(S_t,A_t) + 1$ \tcp*{Increment trace}
    }
}
\end{algorithm}

\section{Gated \Qlearning}

Variable \Qlambda offers a large, untapped space of potential multistep \Qlearning methods.
To address the original problem of balancing trace preservation with off-policy bias mitigation, we focus on a specific subset that we name \emph{Gated \Qlearning}.

Gated \Qlearning leverages the variable nature of $\lambda_t$ to target and attenuate credit along \emph{exploratory} trajectories.
Unlike Watkins' \Qlambda, it does this in a smooth and forgiving way;
a trajectory earns only partial credit for each non-greedy action.
It turns out that this strategy achieves a pure interpolation between Watkins' and Peng's \Qlambda, yet without using any importance sampling.
We call this method \emph{Gated \Qlambda} and present it as the main algorithmic contribution of our paper (see \Cref{sub:gated_qlambda}), but briefly discuss an \nsteptext variant as well (see \Cref{sub:nstep_gated_qlearning}).

\subsection{Gated \Qlambda}
\label{sub:gated_qlambda}

We can conceptualize credit assignment in \Qlambda methods as productive errors flowing backward in time to update previous value estimates, like water through a series of pipe segments.
In any segment, the value of $\lambda$ determines the instantaneous flow rate:
what proportion of the water passes through and affects downstream values.
A maximum value of $\lambda=1$ allows unconstrained flow, flooding the value estimates with noise (high variance).
A minimum value of $\lambda=0$ shuts off the flow entirely, starving the pipeline of precious water (high bias).
Ideally, these two concerns would be balanced.

Continuing with this metaphor, non-greedy actions ``pollute'' the water from a \Qlearning perspective;
exploratory behavior taints the value estimation with bias because such actions do not match the greedy target policy.
However, contaminated water is preferable to no water---the intuitive reason why Peng's \Qlambda outperforms Watkins' \Qlambda, which discards the entire stream at the first sign of pollution.
It is better to selectively regulate the intake of contaminated water to maintain a reasonable flow rate while balancing cleanliness.
This is the core motivation behind Gated \Qlambda.

Formally, we introduce a hyperparameter $\chi \in [0,1]$, which we refer to as the ``gate.''
When an exploratory action is taken, we attenuate the eligibility of the TD error by an additional factor of $\chi$:
\begin{equation*}
    \lambda_t \coloneqq \begin{cases} \lambda & \text{if } \greedy(S_t, A_t), \\ \lambda \chi & \text{otherwise.} \end{cases}
\end{equation*}
This definition makes it clear that Gated \Qlambda achieves a pure mathematical interpolation between the strict, safe updates of Watkins' \Qlambda ($\chi=0$) and the fast, uncorrected updates of Peng's \Qlambda ($\chi=1$).
This targeting mechanism is conceptually analogous to the gating architectures found in Long Short-Term Memory (LSTM) networks \citep{hochreiter1997long}, Gated Recurrent Units \citep[GRUs;][]{cho2014learning,chung2014empirical}, and gated attention \citep{xu2015show,dhingra2017gated}, which multiplicatively attenuate the signals flowing through a network.
Just as these architectures smoothly regulate information flow to protect a network's internal memory state, our mechanism smoothly regulates error flow to protect the action-value estimates.
However, \emph{unlike} these architectures, the gate here is not a learnable parameter and remains independent of the agent's function approximator.

To understand the theoretical implication of this mechanism, let $k_{t:t+i} \in \{0, \dots, i\}$ denote the total number of non-greedy actions taken from time step $t+1$ through $t+i$.
The forward-view return of Gated \Qlambda can be expressed as:
\begin{equation}
    \label{eq:forward_gated}
    G^{\lambda\,\text{(Gated)}}_t = Q(S_t,A_t) + \delta'_t + \sum_{i=1}^\infty (\gamma \lambda)^i \chi^{k_{t:t+i}} \delta_{t+i}
    \,.
\end{equation}
This perspective clearly illustrates that the multistep error decays unconditionally by $(\gamma \lambda)^i$ as in Peng's method, but is further attenuated by an additional factor of $\chi$ for every non-greedy choice made along the trajectory.
While this forward view provides theoretical clarity regarding the geometric accumulation of the gate, we note that practical implementations will rely on either the equivalent recursive formula in \Cref{eq:recursive_general} or standard backward-view eligibility traces as detailed in \Cref{alg:gated_qlambda}.

\subsection{\nsteptext Gated \Qlearning}
\label{sub:nstep_gated_qlearning}

As an aside, we note that we can derive an \nsteptext return target that applies the same gating mechanism, and thus our idea is not specific to \lambdareturns.
We achieve this by truncating the Gated \Qlambda target in \Cref{eq:forward_gated} to just $n$ steps and then setting $\lambda=1$.
This yields the following return estimate:
\begin{equation*}
    \nstep{n}_t \coloneqq Q(S_t,A_t) + \delta'_t + \sum_{i=1}^{n-1} \gamma^i \chi^{k_{t:t+i}} \delta_{t+i}
    \,.
\end{equation*}
Although this \nsteptext estimator is very interesting, we do not pursue it further in this work.
We imagine, though, that it could be beneficial to large-buffer deep RL algorithms in the DQN family \citep{mnih2015human,hessel2018rainbow}, which commonly use biased, uncorrected \nsteptext returns to improve sample efficiency.
A slight drawback to our \nsteptext estimate is that its computational cost scales linearly with $n$ due to the individual TD errors.
However, with massive hardware parallelization, the cost may not be noticeable for typical values of $n$.

\section{Hyperparameter Study}

Before we formally analyze Gated \Qlambda in \Cref{sec:analysis}, we conduct a focused hyperparameter sweep to gain an intuition for how the gating mechanism impacts credit assignment.
Code to reproduce this experiment is available online.\footnote{
    \url{https://github.com/brett-daley/gated-q-learning}
}

For our test environment, we adapt the 19-state random walk from \citet[][Example~7.1]{sutton2018reinforcement}.
This environment has 19 linearly connected states and two actions to move left or right.
The agent starts each episode in the central state.
The agent's behavior policy chooses either action with equal probability.
Reaching the extreme ends of the walk yields a $-1$ or $+1$ reward (left and right, respectively) and terminates the episode.
The simple, linear topology of this environment is ideal for isolating and measuring credit assignment.

Because the traditional random-walk experiment is set up for on-policy prediction, we must modify it for off-policy control.
We first apply a slight discount factor of $\gamma = 0.99$, to incentivize the agent to earn rewards expediently.
We then calculate the optimal action-value function, $q_*$.
The key difference in our setup is that we train the agents for a fixed number of steps (500) instead of episodes, to capture the initial learning speed of the agent.
Each agent's value function is initialized with negligible Gaussian noise ($\sigma = 10^{-9}$) to break ties at the start.
We record the root-mean-square (RMS) error between $Q$ and $q_*$ on every time step.
Note that if we trained for \emph{too} long, the results would be biased in favor of Watkins' \Qlambda because it has no asymptotic bias.
This would fail to capture the phenomenon we really want to study:
the ability for trace preservation across off-policy trajectories to accelerate learning in spite of the bias.
We thus want to examine the early phase of training where trace preservation can be a key performance differentiator.

\begin{wrapfigure}{R}{0.5\textwidth}
    \centering
    \includegraphics[width=0.95\linewidth]{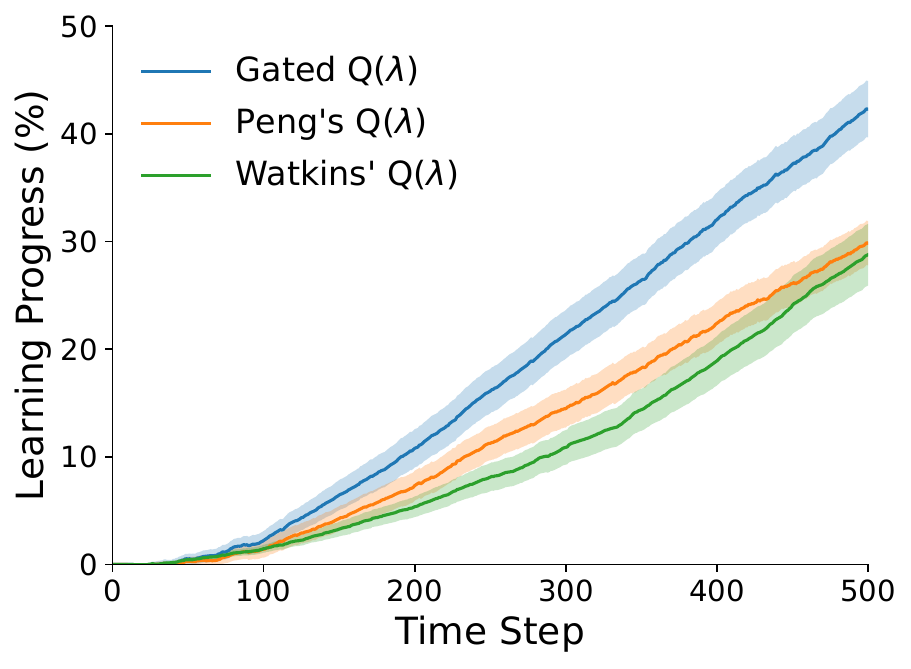}
    \caption{Learning curves for Gated, Peng's, and Watkins' \Qlambda methods with their respective best hyperparameters. Learning progress is defined as the overall RMS error reduction, where 0\% represents no progress and 100\% represents perfect convergence to $q_*$. Results are averaged over 300 random seeds each. The shading represents 95\% confidence intervals.}
    \label{fig:rw_curves}
\end{wrapfigure}

We evaluate Gated \Qlambda with eligibility traces, described in \Cref{alg:gated_qlambda}, by sweeping over $\alpha$, $\lambda$, and $\chi$.
We divide each axis into 21 uniform values from $0$ to $1$.
This results in 9,261 hyperparameter combinations, each averaged across 300 random seeds, for a total of almost 3~million independent trials.
Note that the extremal values of $\chi = 0$ and $\chi = 1$ correspond to Watkins' and Peng's, respectively, so the vast majority of evaluated configurations represent new and previously untested variations of \Qlambda.

\begin{wraptable}{R}{0.5\textwidth}
    \caption{Hyperparameters used in \Cref{fig:rw_curves}.}
    \label{tab:rw_hyperparameters}
    \begin{center}
    \begin{tabular}{rccc}
        \toprule
        Method & $\alpha$ & $\lambda$ & $\chi$ \\
        \midrule
        Watkins' \Qlambda & $1.00$ & $0.95$ & $0.00$ \\
        Gated \Qlambda & $0.95$ & $1.00$ & $0.45$ \\
        Peng's \Qlambda & $1.00$ & $0.70$ & $1.00$ \\
        \bottomrule
    \end{tabular}
    \end{center}
\end{wraptable}

Let us consider how the results are theoretically affected by preserving traces.
Each time the agent moves right, it receives an optimal experience.
Clearly, this experience should be reinforced, and once moving right is greedy, all of the agents weight its influence on past state-action pairs by the same value: $\lambda$.
The sole difference in the methods lies in the case where the agent moves \emph{left}---a suboptimal exploratory move.
If the agent were to move right again afterwards, the methods that do not cut traces will ultimately reinforce the correct action in spite of this exploration.
This is the principal mechanism that makes Peng's \Qlambda learn faster in practice.
However, if the agent \emph{does} end up receiving the negative reward on the left, it will reinforce a bad experience that would never be taken by the optimal policy, incurring bias.

We generate learning curves for each hyperparameter configuration.
We invert and normalize the RMS errors to instead measure \emph{learning progress}, where 0\% represents the initial error (no learning) and 100\% represents perfect convergence to $q_*$ (optimal performance).
We take the area under the curve (AUC) of each learning curve to produce a single summarizing scalar.
The curves for the best-performing hyperparameter selection for each method (with $\chi$ fixed at $0$ or $1$ for Watkins' and Peng's, respectively; see \Cref{tab:rw_hyperparameters}) are plotted in \Cref{fig:rw_curves}.

Collectively, the AUC scalars form a 3-D cube of $21 \times 21 \times 21$ voxels.
We identify the globally optimal point at $\alpha=0.95$, $\lambda=1.00$, and $\chi=0.45$.
We then cross-section the cube through this point along orthogonal planes to generate three 2-D heatmap visualizations in \Cref{fig:rw_heatmaps}.

\begin{figure}[t]
    \centering
    \includegraphics[width=\linewidth]{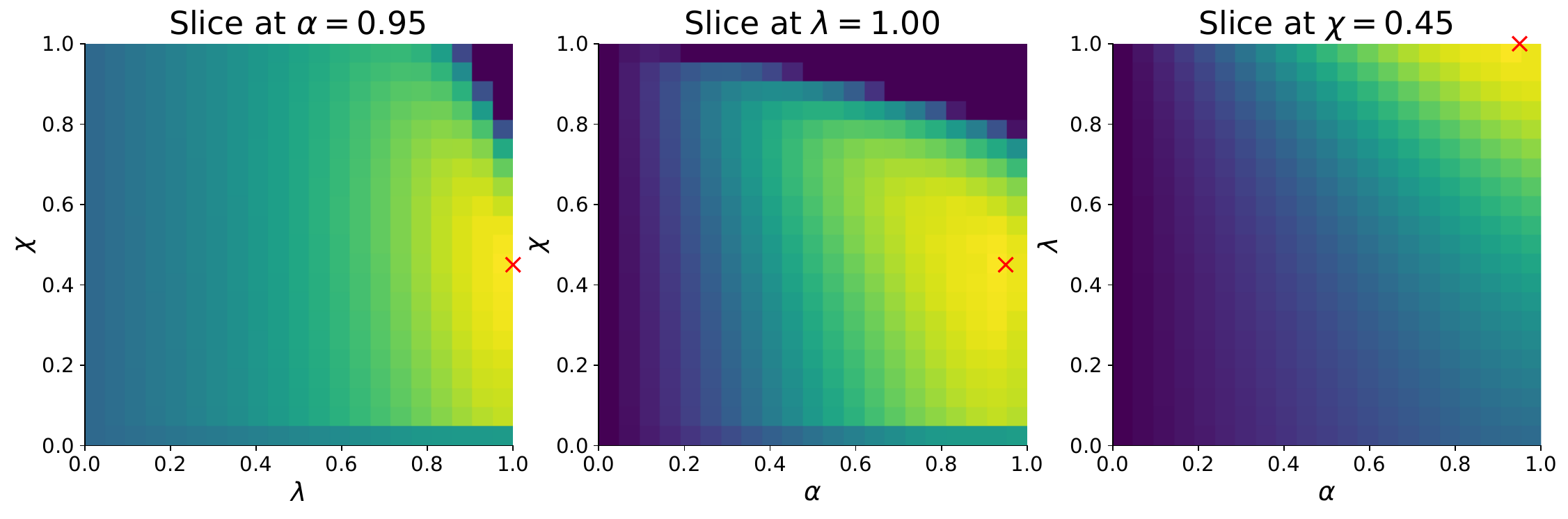}
    \caption{Cross-sectional slices of the 3-D heatmap generated for the random walk. Each slice passes through the coordinate of the best hyperparameters found by the search, indicated by a red~X (see \Cref{tab:rw_hyperparameters}). The colors indicate the area under the curve achieved by each hyperparameter configuration, with warmer colors indicating larger values. Each voxel is averaged over 300 trials.}
    \label{fig:rw_heatmaps}
\end{figure}

This experiment reveals three major insights:
(i) Gated \Qlambda learns significantly faster than both baselines;
(ii) the gating mechanism enables the safe use of higher $\lambda$-values;
and (iii) performance is remarkably robust across a wide range of intermediate gate values ($\chi \in [0.2, 0.6]$).
Clearly, an intermediate $\chi$ improves the best performance in this task.
This gives merit to the idea that balancing the bias-variance trade-off through a gating mechanism is advantageous.

\section{Analysis}
\label{sec:analysis}

In this section, we identify and analyze a general value-function operator underpinning the variable \Qlambda method introduced in \Cref{sec:variable_qlambda}.
Our analysis here is not intended to capture the stochastic behavior of the online eligibility-trace method presented in \Cref{alg:gated_qlambda}.
Instead, we primarily wish to understand the properties of Gated \Qlambda's forward-view return target.\footnote{
    Standard results in stochastic approximation \citep[e.g.,][Ch.~4--5]{bertsekas1996neuro} suggest that a stochastic algorithm with a corresponding contractive operator converges to the same fixed point, given appropriate step-size schedules and sufficient exploration.
    Formally establishing this here is left for future work.
}
This is especially relevant for deep RL, where value estimates are typically produced by a frozen target network and backward-view eligibility traces are generally not used.\footnote{
    We refer the reader to \citet[][Ch.~1, 7]{daley2025multistep} for a detailed account.
    Forward-view returns are more compatible with minibatch sampling, which draws transitions out of temporal order.
    Although some recent algorithms have revisited backward-view eligibility traces for deep RL, they currently tend to be less stable and sample-efficient than replay methods.
}
Operator theory provides exactly the right tool to understand how the gating mechanism addresses the trade-off between expected convergence speed and off-policy bias.

\subsection{Variable \Qlambda Operator}

We start by deriving $\op \colon \RSA \to \RSA$, a general operator capable of expressing the variable \Qlambda methods introduced in \Cref{sec:variable_qlambda}, including Watkins' \Qlambda, Peng's \Qlambda, and our Gated \Qlambda.
Here, vectors represent value functions (e.g., $\vq \in \RSA$ represents $Q$), where each element corresponds to the value estimate for a particular state-action pair.
The operator acts on the state-action value estimates in the update $\vq \gets \op \vq$, thus producing the expected value of the targets $\tilde{G}^\lambda_t$ from \Cref{eq:lambda_variable} for each state-action pair.
Similarly, if we were to set $\alpha$ to be very small and hold $Q$ fixed while executing \Cref{alg:gated_qlambda} and averaging updates on the side, then the net update would be approximately equal to $\op \vq - \vq$.
This is the classic forward-backward equivalence of eligibility traces \citep[][Sec.~7.4]{sutton1998introduction}, allowing us to analyze the expected behavior of the return target and the eligibility traces purely through the lens of $\op$.

To construct $\op$, we first define a few fundamental operators based on the bipartite MDP decomposition from \citet[][Ch.~2]{daley2025multistep}.
Let $P \colon \RS \to \RSA$ be the transition dynamics operator.
We also define a policy operator $E_\pi \colon \RSA \to \RS$ which computes the expected state values under any policy $\pi$.
Let $E_*$ denote the greedy expectation, $(E_* q)(s) \coloneqq \max_{a \in \A} q(s,a)$, which is nonlinear.
Finally, to simplify the analysis, we introduce an expansion operator $J \colon \RS \to \RSA$, which broadcasts a state's value to each of its action values.
Note that the reward function is itself treated as a value function, $\vr \in \RSA$, to which dimensionally compatible operators can be applied.
Here and throughout, the unsubscripted identity matrix $\mI$ has dimension $\abs{\SA} \times \abs{\SA}$.

With these definitions, the Bellman operator for action values is expressed as
\begin{equation}
    \label{eq:bellman_op}
    T \vq \coloneqq \vr + \gamma \mP E_* \vq
    \,.
\end{equation}
Note that we bold both value functions and \emph{linear} operators whenever they appear as such algebraic quantities.
Next, we define the trace-decay matrix $\mLambda$.
Because the trace-decay parameter $\lambda(s, a)$ can vary by state-action pair, $\mLambda$ is a diagonal matrix of size $\abs{\SA} \times \abs{\SA}$.
The elements are defined by
\begin{equation*}
    \emLambda_{i,j} \coloneqq \begin{cases}
        \lambda(s_i,a_i) & \text{if $i=j$} \,, \\
        0 & \text{if } i \neq j \,,
    \end{cases}
    \quad
    \forall \, i, j \in \{1, \dots, \abs{\SA}\}
    \,.
\end{equation*}
It is crucial to note that because methods like Gated \Qlambda define $\lambda(s, a)$ as a function of the current value estimates, $\mLambda$ is an implicit function of $Q$ in general.
This theoretically makes it a nonlinear operator.
However, because we analyze the setting where $\lambda(s,a)$ is computed from a fixed value estimate under a fixed behavior policy, we can treat $\mLambda$ as a linear operator.
This is a standard assumption that makes the analysis tractable while mirroring the stale-target setting common in modern RL architectures.

\begin{restatable}{prop}{opdef}
    \label{prop:op_def}
    The operator $\op$ for variable \Qlambda has the closed-form definition
    \begin{equation}
        \label{eq:lambda_op_def}
        \op \vq \coloneqq (\mI - \gamma \mP \mE_b \mLambda)^{-1} (T \vq - \gamma \mP \mE_b \mLambda \mJ E_* \vq)
        \,.
    \end{equation}
\end{restatable}

\begin{proof}
    See \Cref{app:prop_op_def}.
    The result follows by converting the expectation of the recursive target in \Cref{eq:recursive_general} to operator form, and then manipulating it algebraically to isolate $\op$.
\end{proof}

\Cref{eq:lambda_op_def} reveals that the $\op$ operator takes the generic form $\mZ^{-1}(\vy + X \vq)$.
Here, the matrix $\mZ^{-1} = (\mI - \gamma \mP \mE_b \mLambda)^{-1}$ represents the expected eligibility trace for every state-action pair under the behavior policy $b$ and the chosen decay scheme $\mLambda$. The inner vector term $(\vy + X \vq) = \vr + \gamma \mP E_* \vq - \gamma \mP \mE_b \mLambda \mJ E_* \vq$ fills in the complementary gaps of the decayed trace with bootstrapped state-value estimates derived from $\max_{a \in \A} Q(s,a)$.

\subsection{Contraction Rate}

To quantify the expected error reduction achieved by a single application of $\op$, our next step is to prove that it is a contraction mapping.
Let $\maxnorm{\cdot}$ denote the maximum norm of a vector or matrix (i.e., the maximum row sum of the absolute values).
We formally define a contraction mapping with respect to this norm below.
\begin{definition}
    \label{def:contraction}
    A value-function operator $H \colon \RSA \to \RSA$ is a contraction mapping if and only if there exists a constant $\beta \in [0,1)$ such that, $\forall \, \vq_1, \vq_2 \in \RSA$,
    \begin{equation*}
        \maxnorm{H \vq_1 - H \vq_2} \leq \beta \maxnorm{\vq_1 - \vq_2}
        \,.
    \end{equation*}
\end{definition}
We seek to find a \emph{contraction modulus} for $\op$: a constant $\beta$ that satisfies \Cref{def:contraction}.
To facilitate this, let $\vone_\S \in \RS$ and $\vone_{\SA} \in \RSA$ denote the all-ones vectors and $\mI_\S$ the $\abs{\S} \times \abs{\S}$ identity.
\begin{restatable}{theorem}{contraction}
    \label{th:contraction}
    Let $\vc \coloneqq \mE_b \mLambda \vone_{\SA}$, with entries $c(s) = \sum_{a \in \A} b(a|s) \lambda(s,a)$, and $\cmin \coloneqq \min_{s \in \S} c(s)$.
    The operator $\op$ is a contraction mapping with modulus
    \begin{align*}
        \beta &= \frac{\gamma (1-\cmin)}{1 - \gamma \cmin}
        \,.
    \end{align*}
\end{restatable}

\begin{proof}
    See \Cref{app:th_contraction}.
\end{proof}

The value $\cmin$ represents the smallest state-conditional expected trace-decay value across all states.
Because $\beta$ is monotonically decreasing with $\cmin$, this formalizes the intuition that the trace preservation inherent to methods like Peng's \Qlambda is indeed beneficial to convergence speed in expectation.
When we apply a constant value of $\lambda$ to all state-action pairs, the formula gracefully collapses back to
$\beta = \gamma (1 - \lambda) \mathbin{/} (1 - \gamma\lambda)$,
which exactly matches the known contraction modulus for standard TD($\lambda$) and Peng's \Qlambda \citep[see, e.g.,][Proof of Prop.~3.11]{daley2024averaging}.

For Gated \Qlambda, the algorithm's trace decay depends on whether a greedy action is taken: $\lambda$ if so and $\lambda \chi$ if not.
If the behavior policy has probability $p_g(s)$ of acting greedily in state $s$, then the expected trace value is
$c(s) = p_g(s) \cdot \lambda + (1-p_g(s)) \cdot \lambda \chi$.
This means $\cmin$ is determined by the \emph{least-greedy} state:
the state where the agent currently explores the most.\footnote{
    If the agent is executing an $\epsilon$-greedy policy, then exploration is state-wise uniform and we simply substitute the probability of taking a greedy action:
$p_g(s) = 1 - \epsilon + \epsilon \mathbin{/} \abs{\A}$.
}
This implies that a gate of $\chi < 1$ actually \emph{reduces} the guaranteed convergence speed, making the method more like Watkins' \Qlambda.
However, this can be compensated for by increasing $\lambda$ beyond what would be considered safe for Peng's \Qlambda, as we show with the fixed-point analysis in the next subsection.

\subsection{Fixed Point and Off-Policy Bias}

We have now established that $\op$ is a contraction mapping.
By the Banach fixed-point theorem \citep{banach1922sur}, $\op$ admits a \emph{unique} fixed point, to which repeated application of the operator converges:
i.e., $\lim_{i \to \infty} (\op)^i \vq$ exists and is the same for every $\vq \in \RSA$.
We must now identify the nature of this fixed point.

We once again leverage the recursive formula to unpack the components of the operator.
By doing so, we find that the algorithm implicitly evaluates a composite policy that interpolates between the behavior and (greedy) target policies.
At the fixed point $\vq = \op \vq$, $E_* \vq = \Egreedy \vq$, where $\pi_{\text{greedy}}$ is greedy with respect to $\vq$.
\begin{restatable}{theorem}{fixedpoint}
    \label{th:fixed_point}
    Let $\pimix$ be the mixture policy defined by the following policy operator:
    \begin{equation*}
        \Emix \coloneqq \mE_b \mLambda + (\mI_\S - \mE_b \mLambda \mJ) \Egreedy
        \,.
    \end{equation*}
    The unique fixed point of operator $\op$ is exactly the action-value function of $\pimix$:
    \begin{equation*}
        \qmix \coloneqq (\mI - \gamma \mP \Emix)^{-1} \vr
        \,.
    \end{equation*}
\end{restatable}

\begin{proof}
    See \Cref{app:th_fixed_point}.
\end{proof}

For every choice of the $\lambda(s,a)$ function, variable \Qlambda converges to the value function corresponding to a policy which mixes the behavior and greedy policies.
For clarity, we show the point-wise definition of $\pimix$ which is derived directly from the operators' definitions:
\begin{equation*}
    \pimix(a|s) \coloneqq b(a|s) \lambda(s,a) + (1 - c(s)) \pi_{\text{greedy}}(a|s)
    \,,
\end{equation*}
where $c(s)$ is the state-conditional expected trace decay from \Cref{th:contraction}.
\citet[][Th.~2]{kozuno2021revisiting} proved that Peng's \Qlambda converges to the value function of a mixture policy formed by the convex mixture $\lambda b(a|s) + (1-\lambda) \pi_{\text{greedy}}(a|s)$, which matches our result above when $\lambda(s,a) = \lambda$ for all $(s,a)$.
Thus, our \Cref{th:fixed_point} is a pure generalization.

This result is fascinating in that it formalizes the precise way in which variable \Qlambda methods trade off trace preservation with off-policy bias.
In particular, on-policy values seep into the return estimation whenever we simultaneously have $\lambda(s,a) > 0$ and $b(a|s) > 0$ for a non-greedy pair $(s,a)$, indicating that any amount of preserved trace during exploration contributes bias.
This is because such a pair has $\pimix(a|s) = b(a|s) \lambda(s,a) > 0$, whereas $\pi_{\text{greedy}}(a|s) = 0$, altering the fixed point because the mixture policy is no longer greedy.
This affirms that Watkins' \Qlambda is the only way to truly eliminate off-policy bias by setting $\lambda(s,a) = 0$ whenever $\pi_{\text{greedy}}(a|s) = 0$.
Nevertheless, biased \Qlambda methods can still converge to a \emph{nearby} value function, and converge to it much faster than Watkins' \Qlambda, making them practically useful.
Adapting $\lambda$ based on the state-action pair provides the exact capability to adjust where and how much off-policy bias is added to the return estimation.
Gated \Qlambda offers this capability via the gating mechanism $\chi \in [0,1]$ which only targets exploratory actions, enabling users to fine-tune the degree of off-policy bias without negatively impacting unbiased greedy actions.

\section{Conclusion}

We introduced Gated \Qlearning, a novel algorithmic framework that resolves the 30-year tension between the aggressive trace-cutting of Watkins' \Qlambda and the uncorrected bias of Peng's \Qlambda.
By formalizing partial off-policy corrections through a state-action-dependent gating mechanism, we provide a new continuum of algorithms that finely balances long-term credit assignment with off-policy bias.
Our theoretical analysis proves that this strategy always guarantees a contraction mapping and a unique fixed point, with an exactly characterized (but generally nonzero) fixed-point bias.
Key limitations include the newly added gate hyperparameter $\chi$, which in theory must be tuned in conjunction with $\lambda$.
However, the wide plateau near $\chi \approx 0.5$ in \Cref{fig:rw_heatmaps} suggests a favorable margin of error when setting this hyperparameter, at least in the tested environment.
Another limitation is that we did not make any direct empirical comparisons with importance-sampled estimators like Retrace \citep{munos2016safe}, as we focused specifically on greedy \Qlearning algorithms.
Evaluating the relative effectiveness of these two distinct off-policy corrections remains an important direction for future work.
Finally, our analysis did not consider the variance of the return estimates due to the complexity of the problem, though it seems likely that forgoing importance-sampling ratios is a significant boon to variance reduction.

Although not evaluated in the deep RL setting, Gated \Qlearning can seamlessly integrate into existing agents, including trajectory-based methods that use multistep returns, such as asynchronous \nsteptext \Qlearning \citep{mnih2016asynchronous}, DQN($\lambda$) \citep{daley2019reconciling}, and PQN \citep{gallici2025simplifying}, as well as minibatch-replay methods in the DQN family \citep{mnih2015human,hessel2018rainbow} that use \nsteptext returns (recall \Cref{sub:nstep_gated_qlearning}).
We see significant potential to improve performance in these algorithms, as off-policy bias tends to be extreme under experience replay.
Here, Gated \Qlearning can serve as a simpler and more efficient alternative to importance sampling, though we make no claim of superior performance yet.


\appendix


\bibliography{main}
\bibliographystyle{imports/rlj}

\beginSupplementaryMaterials
\section{Proofs}

This section contains the full mathematical proofs omitted from the main paper for exposition ease.

\subsection{Proof of \Cref{prop:op_def}}
\label{app:prop_op_def}

\opdef*

\begin{proof}
    We first rewrite the recursive $\lambda$-return formula from \Cref{eq:recursive_general} as
    \begin{equation*}
        \tilde{G}^\lambda_t = R_{t+1} + \gamma V(S_{t+1}) + \gamma \lambda_{t+1} \Bigl(\tilde{G}^\lambda_{t+1} - V(S_{t+1}) \Bigr)
        \,.
    \end{equation*}
    To convert this to an operator equation, we must convert the expected values of these quantities to their matrix or vector equivalents:
    \begin{equation}
        \label{eq:op_def_recursive}
        \op \vq = T \vq + \gamma \mP \mE_b \mLambda (\op \vq - \mJ E_* \vq)
        \,.
    \end{equation}
    This is because the first term is the Bellman operator from \Cref{eq:bellman_op}, the time-step shift and decay are handled by multiplying $\gamma \mP \mE_b \mLambda$, and the inner term references the operator itself again, minus the state-value estimate (which must be expanded by $\mJ$ to match the $\SA$ dimension).

    We conclude by solving \Cref{eq:op_def_recursive} algebraically for $\op$.
    Rearranging gives
    \begin{equation}
        \label{eq:fixed_point_start}
        (\mI - \gamma \mP \mE_b \mLambda) \op \vq = T \vq - \gamma \mP \mE_b \mLambda \mJ E_* \vq
        \,,
    \end{equation}
    and left-multiplying both sides by $(\mI - \gamma \mP \mE_b \mLambda)^{-1}$, which exists because $\maxnorm{\gamma \mP \mE_b \mLambda} \leq \gamma < 1$, yields \Cref{eq:lambda_op_def},
    completing the derivation.
\end{proof}

\subsection{Proof of \Cref{th:contraction}}
\label{app:th_contraction}

\contraction*

\begin{proof}
Let $\vq_1, \vq_2 \in \RSA$.
To simplify notation, we define the differences
$\Delta \op \vq \coloneqq {\op \vq_1 - \op \vq_2}$,
$\Delta \vq \coloneqq {\vq_1 - \vq_2}$,
and $\Delta \vv \coloneqq {E_* \vq_1 - E_* \vq_2}$.
Note that $\Delta T \vq = \gamma \mP \Delta \vv$.
From substitution into the recursive definition of $\op$ in \Cref{eq:op_def_recursive} and by linearity, we have
\begin{align*}
    \Delta \op \vq &= \gamma \mP \Delta \vv + \gamma \mP \mE_b \mLambda (\Delta \op \vq - \mJ \Delta \vv) \\
    &= \gamma \mP \mE_b \mLambda \Delta \op \vq + \gamma \mP (\mI_\S - \mE_b \mLambda \mJ) \Delta \vv
    \,.
\end{align*}
To form an upper bound, we apply the \emph{element-wise} absolute value to each vector and invoke the triangle inequality.
We then use the fact that $\abs{\vx} \leq \maxnorm{\vx} \vone$ holds element-wise.
Additionally, observe that $(\mI_\S - \mE_b \mLambda \mJ) \vone_\S = \vone_\S - \mE_b \mLambda \vone_{\SA} = \vone_\S - \vc$.
Because $\mE_b \mLambda \mJ$ is diagonal with entries $c(s) \in [0,1]$, the matrix $\mI_\S - \mE_b \mLambda \mJ$ is nonnegative, so all of the matrices here have nonnegative components and can be safely pulled out of the absolute value:
\begin{align*}
    \abs{\Delta \op \vq} &\leq \gamma \mP \mE_b \mLambda \abs{\Delta \op \vq} + \gamma \mP (\mI_\S - \mE_b \mLambda \mJ) \abs{\Delta \vv} \\
    &\leq \gamma \mP \Bigl( \maxnorm{\Delta \op \vq} \mE_b \mLambda \vone_{\SA} + \maxnorm{\Delta \vv} (\mI_\S - \mE_b \mLambda \mJ) \vone_\S \Bigr) \\
    &= \gamma \mP \Bigl( \maxnorm{\Delta \op \vq} \vc + \maxnorm{\Delta \vv} (\vone_\S - \vc) \Bigr) \\
    &\leq \gamma \mP \Bigl( \maxnorm{\Delta \op \vq} \vc + \maxnorm{\Delta \vq} (\vone_\S - \vc) \Bigr)
    \,,
\end{align*}
where the last step follows because the $\max$ operator is non-expansive, hence $\maxnorm{\Delta \vv} \leq \maxnorm{\Delta \vq}$.

Note that $\maxnorm{\mP} = 1$ because $\mP$ is a stochastic matrix.
Taking the norm of both sides yields
\begin{align}
    \nonumber
    \maxnorm{\Delta \op \vq} &\leq \gamma \Bigl\| \maxnorm{\Delta \op \vq} \vc + \maxnorm{\Delta \vq} (\vone_\S - \vc) \Bigr\|_\infty \\
    \label{eq:maximize_convex_combo}
    &= \gamma \max_{s} \Bigl( c(s) \maxnorm{\Delta \op \vq} + (1 - c(s)) \maxnorm{\Delta \vq} \Bigr)
    \,.
\end{align}
This inequality presents a convex combination of $\norm{\Delta \op \vq}_\infty$ and $\norm{\Delta \vq}_\infty$.
If we assume for a moment that $\norm{\Delta \op \vq}_\infty > \norm{\Delta \vq}_\infty$, the convex combination would be at most $\norm{\Delta \op \vq}_\infty$.
This would imply $\norm{\Delta \op \vq}_\infty \leq \gamma \norm{\Delta \op \vq}_\infty$, which is impossible since $\gamma < 1$ and $\norm{\Delta \op \vq}_\infty > 0$.
Therefore, it must be universally true that $\norm{\Delta \op \vq}_\infty \leq \norm{\Delta \vq}_\infty$.

Because $\norm{\Delta \vq}_\infty$ is the larger of the two quantities, the convex combination is maximized by placing as much weight as possible on it.
This means maximizing $1 - c(s)$ in \Cref{eq:maximize_convex_combo}, which is achieved by setting $c(s) = \cmin$.
Substituting $\cmin$ into the bound yields
\begin{equation*}
    \maxnorm{\Delta \op \vq} \leq \gamma \cmin \maxnorm{\Delta \op \vq} + \gamma (1-\cmin) \maxnorm{\Delta \vq}
    \,,
\end{equation*}
which rearranges to
\begin{equation*}
    \maxnorm{\Delta \op \vq} \leq \underbrace{\frac{\gamma (1-\cmin)}{1 - \gamma \cmin}}_{\beta} \maxnorm{\Delta \vq}
    \,.
\end{equation*}
Because $0 \leq \cmin \leq 1$ by definition of $\lambda(s,a)$, the coefficient $\beta$ is strictly less than $1$ whenever $\gamma < 1$, confirming that $\op$ is a contraction mapping with the stated modulus.
\end{proof}

\subsection{Proof of \Cref{th:fixed_point}}
\label{app:th_fixed_point}

\fixedpoint*

\begin{proof}
    At the fixed point, we must have $\op \vq = \vq$.
    Starting from \Cref{eq:fixed_point_start}, expanding $T \vq$ using \Cref{eq:bellman_op}, and then substituting $\vq$ for $\op \vq$, we obtain
    \begin{align}
        \nonumber
        (\mI - \gamma \mP \mE_b \mLambda) \vq
        &= \vr + \gamma \mP E_* \vq - \gamma \mP \mE_b \mLambda \mJ E_* \vq \\
        \nonumber
        \vq &= \vr + \gamma \mP \mE_b \mLambda \vq + \gamma \mP E_* \vq - \gamma \mP \mE_b \mLambda \mJ E_* \vq \\
        \label{eq:factor_q}
        &= \vr + \gamma \mP (\mE_b \mLambda + \Egreedy - \mE_b \mLambda \mJ \Egreedy) \vq \\
        \label{eq:mixture}
        &= \vr + \gamma \mP \Emix \vq
        \,,
    \end{align}
    where the last step substitutes the policy operator defined in the theorem.
    Note that factoring out $\vq$ in \Cref{eq:factor_q} is justified because $\mLambda$ is fixed (as $Q$ is assumed to be fixed for the operator evaluation) and $\pi_{\text{greedy}}$ is greedy with respect to the fixed point $\vq$.
    This allows us to locally treat the nonlinear operator $E_*$ as the linear operator $\Egreedy$, since $E_* \vq = \Egreedy \vq$.

    \Cref{eq:mixture} reveals that the fixed point takes the standard Bellman form for the previously defined mixture policy.
    To complete the proof, we must verify that $\Emix$ is a valid stochastic matrix, thereby representing a realizable policy.
    This means $\Emix$ must comprise exclusively nonnegative elements and its rows must sum to one.

    We first establish nonnegativity.
    By definition, the policy operators $\mE_b$ and $\Egreedy$, as well as the trace matrix $\mLambda$, contain exclusively nonnegative elements.
    The only term that could theoretically introduce negative values is the subtraction in $\mI_\S - \mE_b \mLambda \mJ$.
    However, the operator $\mE_b \mLambda \mJ$ effectively maps a state to itself with the conditionally expected trace decay $c(s) = \sum_{a} b(a|s)\lambda(s,a)$.
    Because $\lambda(s,a) \in [0,1]$ and the behavior policy $b$ is a valid probability distribution, it is guaranteed that $c(s) \in [0,1]$ for all $s \in \S$.
    Consequently, $\mI_\S - \mE_b \mLambda \mJ$ is a diagonal matrix whose diagonal entries are $1 - c(s) \geq 0$.

    Finally, we verify the row sums by applying the operator to the all-ones vector.
    Recall from the proof of \Cref{th:contraction} that $\vc = \mE_b \mLambda \vone_{\SA}$ and $\vone_\S - \vc = (\mI_\S - \mE_b \mLambda \mJ) \vone_\S$.
    Therefore,
    \begin{align*}
        \Emix \vone_{\SA} 
        &= \mE_b \mLambda \vone_{\SA} + (\mI_\S - \mE_b \mLambda \mJ) \Egreedy \vone_{\SA} \\
        &= \vc + (\mI_\S - \mE_b \mLambda \mJ) \Egreedy \vone_{\SA} \\
        &= \vc + (\mI_\S - \mE_b \mLambda \mJ) \vone_{\S} \\
        &= \vc + \vone_\S - \vc \\
        &= \vone_\S
        \,.
    \end{align*}
    Since applying $\Emix$ to the all-ones vector perfectly recovers the all-ones vector (of the appropriate dimensions), the row sums of the implied transition matrix all equal $1$.
    Coupled with the fact that it has nonnegative components, $\Emix$ represents a valid policy.
    Hence $\maxnorm{\gamma \mP \Emix} = \gamma < 1$, so $\mI - \gamma \mP \Emix$ is invertible, and \Cref{eq:mixture} gives $\vq = \qmix$, completing the proof.
\end{proof}

\end{document}